\documentclass{article}

\usepackage{arxiv}

\usepackage[utf8]{inputenc} 
\usepackage[T1]{fontenc}    
\usepackage{hyperref}       
\usepackage{url}            
\usepackage{booktabs}       
\usepackage{amsfonts}       
\usepackage{nicefrac}       
\usepackage{microtype}      

\usepackage{graphicx}
\usepackage{natbib}
\usepackage{doi}

\usepackage{algorithm}
\usepackage{amsmath}
\usepackage{amsthm} 
\usepackage{amssymb}  
\usepackage{mathtools}
\usepackage{thm-restate}
\usepackage{nicefrac}
\usepackage[noend]{algpseudocode}\algrenewcommand\algorithmicrequire{\textbf{Input:}}
\usepackage{tikz}
\usetikzlibrary{calc,fit,positioning,arrows}

\usepackage{cleveref}       

\newcommand{\abs}[1]{\left\vert #1\right\vert}
\DeclarePairedDelimiter{\norm}{\lVert}{\rVert}
\newcommand{\R}{\mathbb{R}}
\newcommand{\N}{\mathbb{N}}
\newcommand{\scalarprod}[2]{\langle #1, #2\rangle}

\DeclareMathOperator*{\argmin}{arg\,min}
\DeclareMathOperator*{\argmax}{arg\,max}

\newcommand{\F}{\mathcal{F}}
\newcommand{\X}{\mathcal{X}}
\newcommand{\Sc}{\mathcal{S}}
\newcommand{\A}{\mathcal{A}}

\newcommand{\E}{\mathbb{E}}

\DeclarePairedDelimiterX{\infdivx}[2]{(}{)}{%
  #1\;\delimsize\|\;#2%
}

\newtheorem{claim}{Claim}
\newtheorem{theorem}[claim]{Theorem}
\newtheorem{lemma}[claim]{Lemma}
\newtheorem{proposition}[claim]{Proposition}
\newtheorem{corollary}[claim]{Corollary}
\newtheorem{definition}[claim]{Definition}
\newtheorem{assumption}[claim]{Assumption}
\newtheorem*{notation*}{Notation}

\title{Efficiency Separation between RL Methods:\\Model-Free, Model-Based and Goal-Conditioned}

\author{Brieuc Pinon, Raphaël Jungers, Jean-Charles Delvenne\\
ICTEAM/INMA\\
UCLouvain\\
Belgium \\
\texttt{\{brieuc.pinon,raphael.jungers,jean-charles.delvenne\}@uclouvain.be}
}

\renewcommand{\shorttitle}{Efficiency Separation between RL Methods:\\Model-Free, Model-Based and Goal-Conditioned}

\hypersetup{
pdftitle={Efficiency Separation between RL Methods: Model-Free, Model-Based and Goal-Conditioned},
pdfauthor={Brieuc Pinon, Raphaël Jungers, Jean-Charles Delvenne},
}

\begin{document}

\maketitle

\begin{abstract}
    We prove a fundamental limitation on the efficiency of a wide class of Reinforcement Learning (RL) algorithms. This limitation applies to model-free RL methods as well as a broad range of model-based methods, such as planning with tree search.
    
    Under an abstract definition of this class, we provide a family of RL problems for which these methods suffer a lower bound exponential in the horizon for their interactions with the environment to find an optimal behavior. However, there exists a method, not tailored to this specific family of problems, which can efficiently solve the problems in the family.

    In contrast, our limitation does not apply to several types of methods proposed in the literature, for instance, goal-conditioned methods or other algorithms that construct an inverse dynamics model.
\end{abstract}

\section{Introduction}
    
    A significant part of research in Artificial Intelligence is dedicated to creating, analyzing, and evaluating Reinforcement Learning (RL) methods. One goal is to understand when and why some types of methods will work better than others from a statistical and computational point of view. Explaining such differences is of central importance to drive the design of new efficient algorithms.

    A first step in understanding the differences between methods is to abstract them into classes. Two of the main classes of RL methods are model-based and model-free methods. Model-based methods are algorithms that leverage a known or learned model of the environment dynamics \cite{mordatch2020tutomodel}. In contrast, model-free methods do not use such a model.
    
    While the distinction between model-free, for example Q-learning and policy gradient algorithms, and model-based classes is commonly accepted \cite{sutton2018reinforcement}, there exist no agreed-upon general formal definitions of these classes. Authors resort to proposing their own definitions \cite{sun2019model}, or to proving their results on classical algorithms that are representative of these classes \cite{tu2019gap}.

    Several works have studied the relative performance of these classes from a theoretical point of view, and it is cited as an open problem in the survey \citet{levine2020offline}. For tabular Markov Decision Process, \citet{tu2019gap} makes a survey of the existing literature that studied the model-free or model-based methods. They obtain no clear conclusion in favor of one class over another. In the specific problem family of the Linear Quadratic Regulator, \citet{tu2019gap} proves a polynomial separation result.

    To our knowledge, \citet{sun2019model} gives the only result with a gap on the efficiency that is exponential in a relevant quantity to the advantage of model-based over model-free methods. We extend their result in two ways.

    First, we redefine and broaden the class of methods to which a limitation applies. We construct a family of problems that is hard in the horizon, not only for model-free methods, but also for several model-based ones.

    For this family, we also show that there exists an RL method that can efficiently discover the optimal behavior. The second and main contribution with respect to the result of \citet{sun2019model} is that our RL method is not specific to the family of problems, but applies to a much more universal set of problems. In opposition, their algorithm relies crucially on the knowledge of the set of problems present in the family. In general, such knowledge cannot be assumed to be known in practice.

    Our findings are thus the first to implicate a strong efficiency limitation of a wide class of RL methods, while a universal method does not have that limitation. Moreover, the exposed limitation points out that some ideas proposed in the literature could be essential to solve a large set of problems.

    This article is structured as follows. In Section 2, we define the notations and formalize the problems we address. In Section 3, we characterize the class of RL methods on which we will prove a limitation. In Section 4, we state our main Theorem. In Section 5, we demonstrate numerically the Theorem with deep RL algorithms. Finally, we give a summary of our findings and discuss ideas in the literature that could overcome the limitation presented in this paper.

\section{Preliminaries}
        \paragraph{Notations}
        We use $\Delta(\Omega)$ to denote the set of probability measures over a sample space $\Omega$ with an implicitly associated $\sigma$-algebra. We define the function $\delta(.)$ to output $1$ if the condition in its argument is respected, else $0$. For a set $A$, we note $A^*$ the set of all finite sequences of elements in $A$, $\cup_{i\in\N}A^i$. Reference to neural networks initialization refers to the initialization of a multilayer perceptron (MLP), a classical neural network architecture which compose iteratively linear operators and non-linear point-wise activation functions.  We keep implicit the input dimension, output dimension, and number of layers with their respective numbers of hidden units. 
    
        In this paper an RL problem is a finite horizon Markov Decision Process (MDP), which is defined by a horizon $H\in \mathbb{N}$, a state space $\Sc$, an action space $\A$ and an operator $P:(\Sc\times\A)\cup\{\bot\}\rightarrow \Delta(\R\times\Sc)$, which determines the initial state distribution with $P(s_0|\,\bot)$ and the transition dynamics with $P(r,s'|\,s,a)$, where $r$ is the reward and $s'$ is the next state.

        Throughout the paper, the set of actions is binary $\A=\{0,1\}$ and, for some $n\in\N$, the set of states $\Sc$ is contained in $\{(t,x)\in\{0,\ldots, H\}\times\R^n\}$, where $t$ is the time step. Initial states have $t=0$, and $t$ is incremented at each transition by $P$. When the time step $H$ is reached, we say that the state is final and the trajectory ends.

        We note $(s_0,a_0,r_0,s_1,a_1,\ldots,s_H)\sim P^{\pi}$ a trajectory sampled according to the operator $P$ and a policy $\pi:\Sc\rightarrow\Delta(\A)$. We use $\pi^U$ to denote the policy which outputs a uniform distribution over actions.

        An RL method is an algorithm that outputs an optimized policy, $\pi$, to maximize the expected cumulative rewards $\E_{(s_0,a_0,r_0,\ldots,s_H)\sim P^{\pi}}[\sum_{t=0}^{H-1}r_t]$.  To do so, the method can draw and leverage samples of trajectories from the MDP.

\section{A formalization of a large RL class}
        In this section, we provide a definition for a class of RL methods before stating, in the next section, a limitation on their efficiency.
        
        We constrain this class of methods in one main way. For any observed transitions $(s,a,r,s')$, the state $s'$ can only be observed through evaluations of a set of functions. Moreover, this set of functions must respect a symmetry condition.

        As we will illustrate, these constraints are satisfied by a large set of classical RL methods.

        To provide a formal characterization of this class of methods, we will need several definitions.

        In Algorithm \ref{alg:pointer}, we define an encoder and a decoder for states to pointers and pointers to states, respectively. These methods provide a way to obfuscate a state, while still allowing to manipulate that state.
         \begin{algorithm}[t]
            \caption{Encoder and decoder of pointers and states.}
            \label{alg:pointer}
            \begin{algorithmic}
                \State last\_idx $\leftarrow$ -1
                \State list $\leftarrow$[]
                \Function{encode\_state}{$s$}
                    \State last\_idx$\leftarrow$last\_idx+1
                    \State list $\leftarrow$list.push(s)
                    \State \Return last\_idx
                \EndFunction
                \Function{decode\_pointer}{idx}
                    \State \Return list[idx]
                \EndFunction
            \end{algorithmic}
        \end{algorithm}
        We will note $\Bar{s}$ for an index corresponding to state $s$.
        
        In Algorithm \ref{alg:interface}, we define an interface between the RL methods and the RL problem. It leverages the encoder and decoder just defined to obfuscate the states in which we enter. Simultaneously, it maintains and constructs a dataset upon which functions can be defined to evaluate encountered states. This setup allows us to constrain the available information on the states by constraining the set of functions that can be used to evaluate them.
        \begin{algorithm}[t]
            \caption{$\mathrm{Env}:$ an interface for an MDP defined by the transition operator $P$. In input an initialized encoder/decoder of pointers and states as defined in Algorithm \ref{alg:pointer}.}
            \label{alg:interface}
            \begin{algorithmic}
                \State $D\leftarrow []$
                \Function{$\mathrm{Env}\mathrm{\_init}$}{$\F$}
                    \State $s\sim P(\bot)$
                    \State $\Bar{s}\leftarrow\mathrm{encode\_state}(s)$
                    \State \Return $s,\Bar{s},\,\begin{bmatrix}f(s,D)\end{bmatrix}_{f\in\F}$
                \EndFunction
                \Function{$\mathrm{Env}\mathrm{\_step}$}{$\Bar{s}$, $a$, $g$, $\F$}
                    \State $s\leftarrow\mathrm{decode\_pointer}(\Bar{s})$
                    \If{$s$ is not terminal}
                        \State $r,s'\sim P(r,s'|\,s,a)$
                        \State $\Bar{s'}\leftarrow\mathrm{encode\_state}(s')$
                        \State $f_{s'}\leftarrow \begin{bmatrix} f(s',D) \end{bmatrix}_{f\in\F}$
                        \State $D\leftarrow D$.append($(s,g(a,r,f_{s'}))$)
                        \State \Return $s,\,a,\,r,\,\Bar{s'},\, f_{s'}$
                    \ElsIf{$s$ is terminal}
                        \State \Return $\bot$
                    \EndIf
                \EndFunction
                \Function{$\mathrm{Env}\mathrm{\_evaluate\_state}$}{$s$, $\F$}
                    \State \Return $\begin{bmatrix} f(s,D) \end{bmatrix}_{f\in\F}$
                \EndFunction
                \Function{$\mathrm{Env}\mathrm{\_encode}$}{$s$}
                    \State \Return $\mathrm{encode\_state}(s)$
                \EndFunction
                \Function{$\mathrm{Env}\mathrm{\_reset\_data}$}{$ $}
                    \State $D\leftarrow []$
                \EndFunction
            \end{algorithmic}
        \end{algorithm}

        We impose these functions to respect a special symmetry condition upon permutations of the input variables. We define such permutations here.
        \begin{definition}\label{def:perm}
            A \emph{permutation of coordinates} $p:\R^n\rightarrow\R^n$ is a function such that for $x\in\R^n$, $p(x)_{f(i)}=x_i$ for some bijective function $f$ from $\{1\ldots n\}$ to itself.
            
            We will also use interchangeably $p':\N\times\R^n\rightarrow\N\times\R^n$ defined as $p'(s=(t,x))=(t,p(x))$.
        \end{definition}

        From these definitions, we are able to express the assumption that the RL methods part of our class must satisfy.
        \begin{assumption}\label{assum:RL-class-with-interface}
            The RL method takes as unique input the interface defined in Algorithm \ref{alg:interface}. The function $g$ and the sequence of functions $\F$ given in argument to these methods must have the following form $g:\A\times\R\times\R^N\rightarrow X$, and $\F=\{f:\Sc\times(\Sc\times X)^*\rightarrow \R\}^N$ where $X$ is some undefined set and $N$ some natural number. Moreover, for any function $f\in\F$ and any permutation of the coordinates $p$ we must have $f(s,((s_0,x_0),\ldots))=f(p(s),((p(s_0),x_0),\ldots))$.
        \end{assumption}

        To analyze an existing RL method, we must translate this method such that it fits the Assumption, if possible. For example, drawing transitions of the RL problem must be replaced by calls to the interface.
        
        For all the translations to this form that we illustrate, any lower bound on the number of calls of the resulting algorithm to the interface can be translated back to the original RL method as the same lower bound on the number of necessary operations. In some cases, depending on how the original algorithm has been reduced to this form, it can also be translated back as a lower bound on the number of necessary samples of the environment. Giving a statistical bound as well as a computational bound.

        We note that the interface $\mathrm{Env}$ allows not only to sample full trajectories, but also to generate transitions at any encountered state. Thus, it is possible to use common local planning methods that rely on such a generator, for example, local tree search procedures.
        
        The set $\F$ is used as a set of learning algorithms in our translation of existing RL methods to use this interface. The restriction on $\F$ can be intuitively understood as a restriction on the learning algorithms.
        
        Specifically, we work with the assumption that the learning procedure treats the coordinates of the input state vectors symmetrically. This assumption is natural since, without a specific prior about the task at hand, we do not wish to process the different coordinates in particular ways. However, we note that in a practical setting with neural networks, the weights are initialized randomly, which can create asymmetries. Nevertheless, we prove in Appendix \ref{sec:sym-nn} that the distribution of trained neural networks has the demanded symmetries under a natural assumption on their architecture.

        \subsection{Translation of a classical RL method}
            As an example of our formalization, we show how a classical RL method, fitted Q-iteration \cite{riedmiller2005neural}, can be translated to use the interface as defined in Assumption \ref{assum:RL-class-with-interface}. We give two more complete examples in Appendix \ref{sec:translation-alg}: a model-free policy gradient algorithm and a local tree search model-based method. We also explain in the Appendix how other model-based RL methods can be cast to use the Assumption. These examples depict how a much larger set of algorithms could be adapted to this form.

            To clarify the presentation, we use methods as simplified as possible. For example, we only work with one-step RL methods, which use $(s,a,r,s')$ instead of longer sequences. The presented formalism can however be extended to include these multi-step methods. Our results are robust to such changes.

            In Algorithm \ref{alg:fitted-Q-iter-original}, we define the fitted Q-iteration algorithm with deep learning, a common variant of model-free deep Q-learning methods. The Algorithm trains at each iteration a new Q-function on a new dataset following the Bellman equation and the previously trained and fixed $Q$ function. For simplicity, we do not add an exploration term to the policy.
            \begin{algorithm}[t]
                \caption{Fitted Q-iteration (with minimal exploration)}
                \label{alg:fitted-Q-iter-original}
                \textbf{Parameters:} $H$: horizon of the MDP, $\Sc$: state space, $\A$: action space, $K$: number of iterations, $\eta>0$: scalar factor for gradient descent, $I$: number of samples by iteration (divisible by $H$)
                \begin{algorithmic}
                    \Require $P$: the operator corresponding to the MDP to solve
                    \State Initialize a neural network $Q_\theta:\Sc\rightarrow\R^{\abs{A}}$
                    \State $\pi(a|\,s,\,Q)\leftarrow\delta(a=\argmax_{a\in A}Q(s,a))$
                    \For{$k\leftarrow 1,\ldots, K$}
                        \State $\Bar{Q}\leftarrow Q_{\theta}$
                        \State Initialize a neural network $Q_\theta$
                        \For{$i\in 1,\ldots, I/H$}
                            \State $s\leftarrow P(s|\,\bot)$
                            \For{$t\leftarrow 0\ldots H-1$}
                                \State $a\sim \pi(a|\,s,\,\Bar{Q})$ 
                                \State $r,s' \leftarrow P(r,\,s'|\,s,\,a)$
                                \State $\theta\leftarrow \theta - \eta \nabla_\theta \left[Q_\theta(s,a)-r-\max_{a\in\A}\Bar{Q}(s',a)\right]^2$
                                \State $s\leftarrow s'$
                            \EndFor
                        \EndFor
                    \EndFor
                    \State $\pi(a|\,s)\leftarrow\pi(a|\,s,\,Q_\theta)$
                    \State \Return $\pi(a|\,s)$
                \end{algorithmic}
            \end{algorithm}

            In Algorithm \ref{alg:fitted-Q-iter-translated}, we give a translation of the original algorithm using the interface defined in Algorithm \ref{alg:interface}. We specify the function $g$ to construct a dataset and the set of functions $\F$ to learn and apply the Q-functions.
            
            Following the argument formalized in Appendix \ref{sec:sym-nn} about the symmetries existing in neural networks training by gradient descent, the functions defined in $\F$ practically satisfy the symmetry condition of Assumption \ref{assum:RL-class-with-interface}.
            
            \begin{algorithm}[t]
                \caption{Fitted Q-iteration implemented with the interface defined in Algorithm \ref{alg:interface}. The function $\mathrm{chunk}(D,\,I)$ partitions the data sequence $D$ into contiguous subsequences of length $I$ until there is less than $I$ elements left which are put in the last subsequence.}
                \label{alg:fitted-Q-iter-translated}
                \textbf{Parameters:} $H$: horizon of the MDP, $\Sc$: state space, $\A$: action space, $K$: number of iterations, $\eta>0$: scalar factor for gradient descent, $I$: number of samples by iteration
                \begin{algorithmic}
                    \Require $\mathrm{Env}$: the MDP interface with its methods defined in Algorithm \ref{alg:interface}
                    \State $\pi(a|\,Q_s)\leftarrow\delta(a=\argmax_{a\in A}Q_s(a))$
                    \Function{$\mathrm{learnQ}_a$}{$s,\,D$}
                        \State Initialize a neural network $Q_\theta:\Sc\rightarrow\R^{\abs{\A}}$
                        \State $D_1,\ldots,D_{n-1},D_n\leftarrow\mathrm{chunk}(D,\,I)$
                        \For{$(s_d,(a_d,r_d,Q_{s_d'}))\leftarrow D_{n-1}$}
                            \State $\theta\leftarrow \theta - \eta \nabla_\theta \left[Q_\theta(s_d,a_d)-r_d-\max_{a\in\A}Q_{s_d'}(a)\right]^2$
                        \EndFor
                        \State \Return $Q_\theta(s,a)$
                    \EndFunction
                    \State $\F\leftarrow [\mathrm{learnQ}_a]_{a\in \A}$
                    \State $g\leftarrow g(a,r,f_{s'})=(a,r,f_{s'})$
                    \For{$k\leftarrow 1,\ldots, K\cdot I/H$}    
                        \State $s,\Bar{s},Q_s\leftarrow\mathrm{Env}\mathrm{\_init}(\bot)$
                        \For{$t\leftarrow 0,\ldots, H-1$}
                            \State $a\sim \pi(a|\,Q_s)$ 
                            \State $s,a,r,\Bar{s},Q_s \leftarrow\mathrm{Env}(\Bar{s},\,a,\,g,\,\F)$
                        \EndFor
                    \EndFor
                    \State $\pi(a|\,s)\leftarrow\delta(a=\argmax_{a\in A}\mathrm{Env\_eval\_state}(s,[\mathrm{learnQ}_a]_{a\in \A})$
                    \State \Return $\pi(a|\,s)$
                \end{algorithmic}
            \end{algorithm}

\section{Main theorem}
    The main result of this paper states an efficiency gap between a broad group of RL methods and a specific algorithm. We present this algorithm.
    
    Algorithm \ref{def:alg_working} is a simple goal-conditioning method that first samples a dataset of trajectories by drawing actions uniformly. From this dataset, it extracts a state which gives maximal reward when entering it and poses this state as its goal. From the same dataset, a function is learned to predict the action taken from any state given the final state encountered. Finally, it constructs a policy that tries to reach the goal with the action predictor conditioned by the goal.

    \begin{algorithm}[t]
        \caption{A goal-conditioned algorithm.}\label{def:alg_working}
        \textbf{Parameters:} $N$: number of samples, $\alpha$: parameter of the learning algorithm (number of active features)
        \begin{algorithmic}[1]
            \Require $P$: the operator corresponding to the MDP to solve
            \State $D\leftarrow\{\}$
            \For{$i\leftarrow 1,\ldots N$}
                \State $D\leftarrow D\cup\{(s_0,a_0,r_0,s_1,a_1,\ldots,s_H)\sim P^{\pi^U}\}$
            \EndFor
            \State $g\leftarrow \argmax_{s_H} r\quad \text{s.t.}\quad (\ldots,r,s_H)\in D$
            \For{$t\in\{0,\ldots,H-1\}$}
                \State $D^t_{GC}\leftarrow \{((s_t,s_H),a_t)|\,\forall (\ldots, s_t, a_t,\ldots,s_H)\in D\}$
                \State $f^t\leftarrow$ the result of the optimization program \ref{eq:empirical-opt} with dataset $D^t_{GC}$ and parameter $\alpha$.
            \EndFor
            \State \Return $\pi(a|s_t)=f^t(a|s_t,g)$
        \end{algorithmic}
    \end{algorithm}

    The learning algorithm for the action prediction simply minimizes the empirical rate of errors on the dataset. The space of functions in which we learn is a composition of a feature selection, a linear function, and a threshold (to output a binary prediction). The formal mathematical program is
    \begin{equation}\label{eq:empirical-opt}
    \begin{aligned}
        \argmin_{f\in F} \quad & \sum_{((s_t,s_H),a)\in D^t_{GC}} \delta(f(\begin{bmatrix}s_t\\s_H\end{bmatrix})\neq a) \\
        \textrm{s.t.} \quad & \norm{w}_0\leq \alpha,
    \end{aligned}
    \end{equation}
    where $\begin{bmatrix}s_t\\s_H\end{bmatrix}$ denotes the real part of the states concatenated, dataset $D^t_{GC}$ is defined in Algorithm \ref{def:alg_working}, $F$ is the set of linear functions with a threshold, $F=\{x\in \R^{2n}\rightarrow\delta(\scalarprod{w}{x}>0)|\,\forall w\in \R^{2n}\}$, and $w\in\R^{2n}$ is the parameter associated to $f_w$. The condition $\norm{w}_0\leq \alpha$ bounds the number of non-zero weights in the linear function by $\alpha\in\N$.

    We argue that this algorithm is designed for a much more universal set of problems than just for the family of problems used in the proof. This procedure could apply to any goal-reaching problem. The initial policy to explore the dynamics is simply uniform. The space of functions in which we learn our goal-conditioned policy is universal, and only made sufficiently simple to make the proof more straightforward. This claim is supported by the numerical analysis with neural networks presented in the next section.

    In contrast, the family of methods for which we prove an inefficiency is only assumed to respect Assumption \ref{assum:RL-class-with-interface}. This allows methods that are specifically designed to be performant on the family of the proof.

    We have now defined everything needed to state our main result.

    \begin{restatable}{theorem}{mainthm}\label{thm:main_thm}
        There exists a family of finite horizon MDPs such that
        \begin{enumerate}
            \item For any MDP in the family, with probability at least $1-\delta$ (over the sampled trajectories), Algorithm \ref{def:alg_working} outputs an optimal policy for the MDP with a number of samples and number of operations upper bounded by a polynomial in $H$ and $\nicefrac{1}{\delta}$.
            \item For any algorithm satisfying Assumption \ref{assum:RL-class-with-interface} and using $o(2^H)$ calls to the interface (Algorithm \ref{alg:interface}), there exists a problem in the family for which it outputs a suboptimal policy with probability at least $\nicefrac{1}{3}$.
        \end{enumerate}
    \end{restatable}

    The complete proof is given in Appendix \ref{sec-a:main-thm}, it is partially inspired by some of the proofs in \citet{sun2019model}. We give a sketch here with the principal intuitions.

    We define explicitly a family of MDPs satisfying the requirements of the Theorem. We define an MDP of the family by its horizon and by $b$ a binary word that represents a sequence of actions that solves the task. A representation of one element of the family for a small horizon is given in Figure \ref{fig:example-family-RL}.

    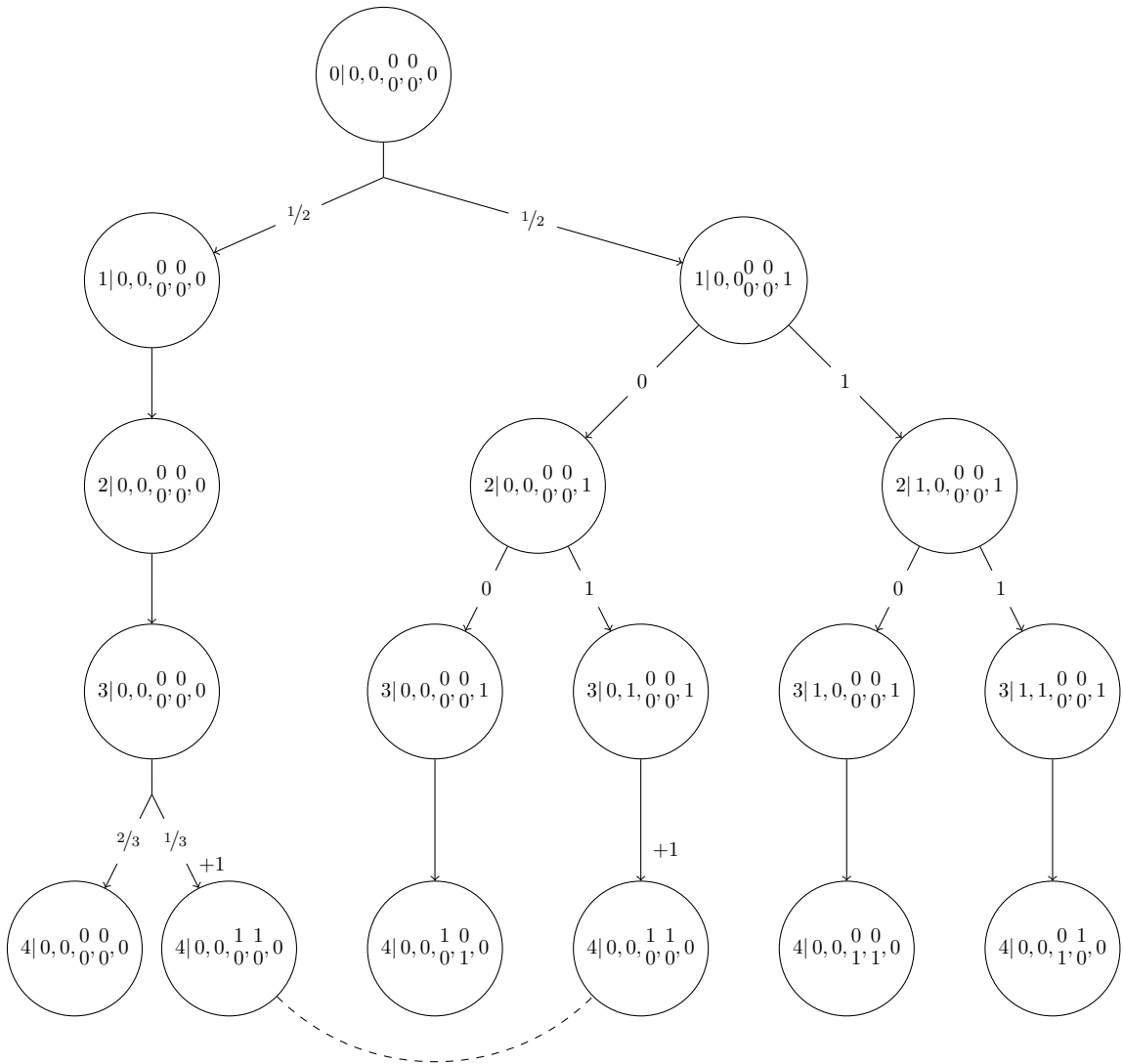
\begin{figure*}[ht!]
        \centering
        \begin{tikzpicture}[sibling distance=40mm,scale=0.35]
            \begin{scope}[every node/.style = {shape=circle, align=center, minimum size=0.7cm, scale=0.8}]

                \node[draw] (0) at (-14,8) {$0|\,0,0,\begin{matrix}0\\0\end{matrix},\begin{matrix}0\\0\end{matrix},0$};
                
                \node[draw] (1) at (0,0) {$1|\,0,0\begin{matrix}0\\0\end{matrix},\begin{matrix}0\\0\end{matrix},1$};
                
                \node[draw] (2) at (-8,-8) {$2|\,0,0,\begin{matrix}0\\0\end{matrix},\begin{matrix}0\\0\end{matrix},1$};
                \node[draw] (3) at (8,-8) {$2|\,1,0,\begin{matrix}0\\0\end{matrix},\begin{matrix}0\\0\end{matrix},1$};
                
                \node[draw] (4) at (-12,-16) {$3|\,0,0,\begin{matrix}0\\0\end{matrix},\begin{matrix}0\\0\end{matrix},1$};
                \node[draw] (5) at (-4,-16) {$3|\,0,1,\begin{matrix}0\\0\end{matrix},\begin{matrix}0\\0\end{matrix},1$};
                \node[draw] (6) at (4,-16) {$3|\,1,0,\begin{matrix}0\\0\end{matrix},\begin{matrix}0\\0\end{matrix},1$};
                \node[draw] (7) at (12,-16) {$3|\,1,1,\begin{matrix}0\\0\end{matrix},\begin{matrix}0\\0\end{matrix},1$};
                
                \node[draw] (8) at (-12,-26) {$4|\,0,0,\begin{matrix}1\\0\end{matrix},\begin{matrix}0\\1\end{matrix},0$};
                \node[draw] (9) at (-4,-26) {$4|\,0,0,\begin{matrix}1\\0\end{matrix},\begin{matrix}1\\0\end{matrix},0$};
                \node[draw] (10) at (4,-26) {$4|\,0,0,\begin{matrix}0\\1\end{matrix},\begin{matrix}0\\1\end{matrix},0$};
                \node[draw] (11) at (12,-26) {$4|\,0,0,\begin{matrix}0\\1\end{matrix},\begin{matrix}1\\0\end{matrix},0$};
                
                \draw[->] (1) -- (2) node[midway,fill=white] {$0$};
                \draw[->] (1) -- (3) node[midway,fill=white] {$1$};
                \draw[->] (2) -- (4) node[midway,fill=white] {$0$};
                \draw[->] (2) -- (5) node[midway,fill=white] {$1$};
                \draw[->] (3) -- (6) node[midway,fill=white] {$0$};
                \draw[->] (3) -- (7) node[midway,fill=white] {$1$};
                \draw[->] (4) -- (8);
                \draw[->] (5) -- (9) node[near end,right] {$+1$};
                \draw[->] (6) -- (10);
                \draw[->] (7) -- (11);
                
                \node[draw] (12) at (-23,0) {$1|\,0,0,\begin{matrix}0\\0\end{matrix},\begin{matrix}0\\0\end{matrix},0$};
                \node[draw] (13) at (-23,-8) {$2|\,0,0,\begin{matrix}0\\0\end{matrix},\begin{matrix}0\\0\end{matrix},0$};
                \node[draw] (14) at (-23,-16) {$3|\,0,0,\begin{matrix}0\\0\end{matrix},\begin{matrix}0\\0\end{matrix},0$};
                \node[draw] (15) at (-26,-26) {$4|\,0,0,\begin{matrix}0\\0\end{matrix},\begin{matrix}0\\0\end{matrix},0$};
                \node[draw] (16) at (-20,-26) {$4|\,0,0,\begin{matrix}1\\0\end{matrix},\begin{matrix}1\\0\end{matrix},0$};
                
                \coordinate (17) at (-23,-20) {};
                \coordinate (18) at (-14,4) {};
                \draw[->] (12) -- (13);
                \draw[->] (13) -- (14);
                \draw[-] (14) -- (17) {};
                \draw[->] (17) -- (15) node[midway,fill=white] {$\nicefrac{2}{3}$};
                \draw[->] (17) -- (16) node[midway,fill=white] {$\nicefrac{1}{3}$}
                                        node[near end,right] {$+1$};
                \draw[-] (0) -- (18) {};
                \draw[->] (18) -- (12) node[midway,fill=white] {$\nicefrac{1}{2}$};
                \draw[->] (18) -- (1) node[midway,fill=white] {$\nicefrac{1}{2}$};

                \draw[dashed] (16) to[in=-135,out=-45] (9);
                
            \end{scope}
        \end{tikzpicture}
        \caption{Example of the family of RL problems for $H=4$ and $b=01$. On the left-hand side, the agent can easily discover the rewarding state. However, the path to reach it is suboptimal. By understanding and using the dynamics on the right-hand side, the algorithm can uncover the optimal path.}
        \label{fig:example-family-RL}
    \end{figure*}

    For the dynamics, at the initial state, there are two possible continuations that are drawn randomly by the RL problem. The left-hand side will give a trajectory over which the agent has no control and ends up randomly either in a non-rewarding state or a rewarding state. On the right-hand side, in that part the agent has full control over the trajectories which can end in a non-rewarding state or a unique rewarding state.

    For algorithms satisfying Assumption \ref{assum:RL-class-with-interface}, the right-hand side has a number of possible end states which grows exponentially in the horizon, and the unique rewarding state becomes hard to reach by purely random actions. Moreover, these algorithms do not have access to relevant information to orient the search: the left-hand side dynamics is independent of $b$, and we show in the proof that the information in the end states of the right-hand side is obfuscated due to the Assumption.

    For Algorithm \ref{def:alg_working}, the left-hand side allows the agent to easily discover a rewarding state and set it as its goal. On the right-hand side, the dynamics is sufficiently simple for its learning procedure to perfectly predict the necessary sequence of actions to reach this goal. Its returned policy will thus act perfectly to reach the rewarding state on the right-hand side. 

\section{Numerical experiments}
    We demonstrate numerically how practical deep RL methods perform on the family of RL problems constructed in the proof. We test Algorithm \ref{def:alg_working} but with neural networks trained by gradient descent as a learning procedure instead. We test a more practical version of the fitted Q-iteration \cite{riedmiller2005neural}, presented in Algorithm \ref{alg:fitted-Q-iter-original}. We also implement and run a classical Actor-Critic method, Proximal Policy Optimization (PPO) \cite{schulman2017proximal}. We refer to Appendix \ref{sec:num-exp} for more details on the implementations.

    We apply these methods to problems of the family with increasing horizons. We sample a dataset of $1000$ trajectories for the goal-conditioned method. For the fitted Q-iteration and PPO methods, we sample $50$ times $1000$ trajectories during training. To measure the success of a method, we check if its returned policy reaches the rewarding state on the right-hand side with $1000$ sampled trajectories, this allows the goal-conditioned algorithm to do some exploration. The results are presented on Figure \ref{fig:numerical-results}.

    \begin{figure}[ht!]
        \centering
        \includegraphics[scale=0.51]{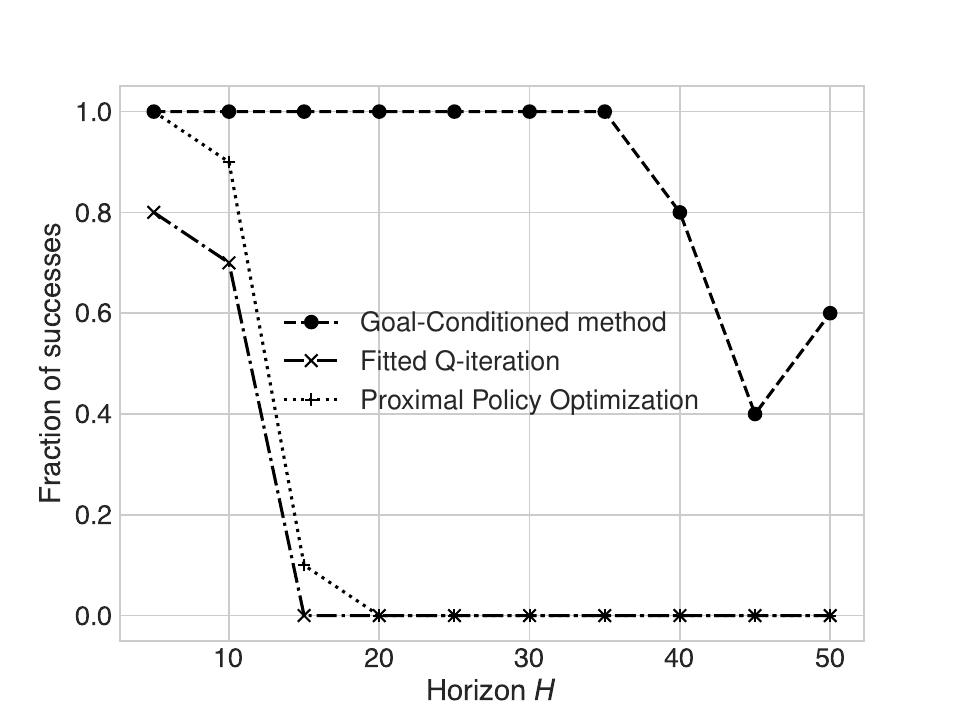}
        \caption{Fraction of successes over 10 runs as a function of the horizon for each method.}
        \label{fig:numerical-results}
    \end{figure}

    The model-free methods quickly fail to solve the task when the horizon is increased. While the goal-conditioning method successfully solves the task for much longer. Both of these observations are suggested by our theoretical findings.

\section{Conclusion and discussion}
    We defined a new class of RL methods that encompasses model-free, such as Q-Learning and policy gradient, and several model-based methods. For this class, we showed an efficiency limitation on a family of problems that are efficiently solvable by an RL method generally applicable to goal-reaching problems.

    The problems in the family feature two different sides with their different respective dynamics but with a common unique rewarding state. The first side has a dynamics that allows an algorithm to easily discover a rewarding state. However, this side has no optimal way to reach that state. The second side possesses an optimal path but it is hard to find it without seeking to reach the rewarding state.
    
    To summarize intuitively our findings, a large class of RL methods will unnecessarily struggle in environments where there exists both: an easy-to-explore but suboptimal-to-use dynamics and a second dynamics which is hard to explore, but easy to navigate if one leverages a known goal that has been discovered in the first dynamics.

    There exist several types of methods proposed in the literature that evade the class of RL methods on which we prove the limitation. We identified the followings:
    \begin{itemize}
        \item Algorithms that learn a smooth model of the dynamics then use backpropagation through the learned model \cite{nguyen1990neural,jordan2013forward,deisenroth2011pilco,grondman2015online,heess2015learning}. These methods elude our definition because the learned model is not treated as a black box by the algorithm. We note that these methods are mostly used in environments where the dynamics are approximately smooth, such as low-level control in robotics.
        
        \item Algorithms that construct a symbolic or semi-parametric model of the dynamics and then apply a symbolic planning algorithm \cite{konidaris2018skills}. This example is related to the previous one, it evades the limitation when the symbolic planner does not treat the model only as a black box generator of transitions but leverages insights in it.
        
        \item Universal value functions \cite{sutton2011horde,schaul2015universal} can be learned to decide which states can be efficiently reached from which states. These algorithms bypass our definition because the value functions take as input the states in the future of the trajectory, and the quantities they learn cannot be replaced trivially by a generator of transitions. 
        
        \item A wide variety of algorithms that learn a link from the future to the past of a trajectory. An example is learning an inverse dynamics model. These algorithms avoid the limitation for the same reason as the universal value functions: the future of the trajectory is directly used to learn the relevant functions, and there is no trivial efficient way to replace the computed quantities with a generator of transitions.

        Inverse dynamics models or goal-conditioning methods learn to predict the action to take given the current state and a state to reach in the future \cite{ghosh2019learning,emmons2021rvs}. There exist variants of this idea, \citet{janner2022planning} proposes to learn to map the current and future state to a full sequence of intermediary states. It is also possible to condition on other information than future states, such as future returns \cite{schmidhuber2019reinforcement,kumar2019reward,chen2021decision,emmons2021rvs}.
        
        These different functions can be efficiently learned with Hindsight Experience Replay, where what is reached in a trajectory is relabelled as a goal in hindsight for training \cite{kaelbling1993learning,andrychowicz2017hindsight}. 
        
        We note however that not all these methods are sound in the presence of uncertainty, as described (and alleviated) in \citet{paster2022you,eysenbach2022imitating,yang2022dichotomy,villaflor2022addressing}.
    \end{itemize}

    Not all of these algorithms necessarily solve efficiently the family of RL problems we defined. Our theoretical result suggests that the ideas present in them could help solve problems otherwise intractable by a large class of classical RL methods.

\bibliographystyle{unsrtnat}
\bibliography{biblio}

\clearpage
\appendix
\section{Main Theorem proof}\label{sec-a:main-thm}
    \mainthm*
    \begin{proof}
    We define the family of RL problems. The family is parametrized by $H$, the horizon, and by a hidden fixed binary word $b$ of length $H-2$. Each state is defined by $3(H-2)+1$ variables. The action space is binary. An example in this family for $H=4$ and $b=01$ is represented on Figure \ref{fig:example-family-RL}.

    We pose some notations to describe the states. Each state is decomposed into three parts $a,\,b,$ and $c$, the part $b$ is further decomposed into parts $u$ and $d$. The time step to which the state belongs is kept implicit in our description. For $i$ between $1$ and $H-2$:
    \begin{itemize}
        \item $s^{a,i}$ the $i$th variable in the first part of the state vector;
        \item $s^{b,u/d,i}$ if $u$, the up $i$th variable in the second part of the state vector, if $d$, the down $i$th variable in the second part of the state vector;
        \item $s^{c}$ the variable of the third part of the state vector.
    \end{itemize}
    Example of the notation for state vector $s$:
    \begin{equation*}
        \left[
        s^{a,1},s^{a,2},\ldots,s^{a,i},\ldots,s^{a,H-2},\begin{matrix}s^{b,u,1}\\s^{b,d,1}\end{matrix},\begin{matrix}s^{b,u,2}\\s^{b,d,2}\end{matrix},\ldots, \begin{matrix}s^{b,u,i}\\s^{b,d,i}\end{matrix},\ldots,\begin{matrix}s^{b,u,H-2}\\s^{b,d,H-2}\end{matrix},s^c
        \right].
    \end{equation*}

    Now we describe the dynamics.
    At step $0$, whatever action is taken, with probability $\nicefrac{1}{2}$, we reach either a state full of zeros, or a state full of zeros except for $s^{c}$ which equals one.

    If we have $s^c=0$, then for the steps $t=1$ to $t=H-1$, the state vector stays zero. At the last step with probability $\nicefrac{1}{3}$ it gives  a reward of $1$ and the state with $s^a=0^{H-2}$, $s^{b,u}=1^{H-2}$, $s^{b,d}=0^{H-2}$ and $s^c=0$. With probability $\nicefrac{2}{3}$, it gives zero-reward and keeps all its variables to zero.

    If $s^c=1$, then from time step $t=1$ to $t=H-2$, the transition from $t$ to $t+1$ encodes $a_t$ into $s_{t+1}^{a,t}$ with zero reward (the other variables keep their values). At the last step, the dynamics fixes $s_H^c=0$, all the variables that encode a past action to zero $s_H^{a}=0^{H-2}$, and $s_{H}^b$ to express if each action taken was correct $\begin{pmatrix}1\\0\end{pmatrix}$ or not $\begin{pmatrix}0\\1\end{pmatrix}$ with respect to a fixed optimal trajectory given by $b$ the binary word parametrizing the problem. In other words, $s_{H}^{b,u,i}=\delta(s_t^{a,i}=b_i)$ and $s_H^{b,d,i}=\delta(s_t^{a,i}\neq b_i)$. At that step, if $s_H^{b,u}=1^{H-2}$, then a reward of $1$ is given. Such that, a reward is obtained only if all the actions taken follow the hidden binary word $b$.

    We prove the first point of the Theorem. Fix any $\delta\in(0,1)$.

    There exists a lower-bound polynomial in $\log \nicefrac{1}{\delta}$ on the number of sampled trajectories, $K$, such that for any horizon, $H$, with probability at least $1-\nicefrac{\delta}{2}$ the rewarding state is discovered with the dynamics on the left-hand side ($s^c=0$).
    
    Moreover, we demonstrate in Lemma \ref{lem:GC-efficient}, that Algorithm \ref{def:alg_working} obtains, for $\alpha=1$ and any $K$ larger than some polynomial in $H$ and $\nicefrac{1}{\delta}$, with probability at least $1-\nicefrac{\delta}{2}$ over the sampled dataset, a goal-conditioned function which predicts correctly the used action on the right-hand side (when $s^c=1$).

    By choosing $K$ sufficiently large and applying the union bound, the probability that one of these events fails is bounded by $\delta$. The policy $\pi(a|s_t)=f^t(a|s_t,g)$ will thus perfectly solve the problem when $s_t^c=1$. When $s_t^c=0$, the policy does not influence the expected reward. The policy is thus optimal.
    
    Moreover, this Algorithm can be efficiently implemented from a computational complexity point of view.

    Now we prove the second claim. We will prove that the part of the problem where $s^c$ equals ones is indistinguishable from an armed-bandit problem family (Definition \ref{def:armed-bandit}) with $2^{H-2}$ arms where only one arm gives a reward. Following Proposition \ref{prop:bandit-hard}, any algorithm using $o(2^H)$ calls will fail to identify the best arm with probability at least $\nicefrac{1}{3}$ on some problem in the armed-bandit problems family. This entails the Theorem.

    In the rest of the proof, we will suppose that we are on the branch where $s^c=1$.
    We prove that all the states at the last steps which have zero-reward are indistinguishable.
    Due to Assumption \ref{assum:RL-class-with-interface}, since these states are final, they cannot be added as direct input in the dataset $D$, such that, information on these states can only be obtained through evaluations of functions $f\in\F$ with one of these states as the first argument.
    
    We prove that these evaluations are identical for all those states. These states only contain non-zero elements in coordinates which are always zero in the states present in the dataset. Moreover, they have exactly the same number of ones in those coordinates. Thus, $f\in\F$, which is symmetric on permutations of the coordinates, cannot distinguish them. Formally, let be two final states on the right-hand side $w,z$, there exists a permutation of the coordinates $p$ such that $w=p(z)$ and $s=p(s)$ for all states $s$ present in the dataset $D$. This implies, for any $f\in\F$,
    \begin{equation*}
    \begin{split}
        f(z,((s_1,y_1),\ldots,(s_N,y_N)))&=f(p(z),((p(s_1),y_1),\ldots,(p(s_N),y_N)))\\
            &=f(w,((s_1,y_1),\ldots,(s_N,y_N))).
    \end{split}
    \end{equation*}

    Thus, non-rewarding final states on the right-hand side are indistinguishable for algorithms satisfying Assumption \ref{assum:RL-class-with-interface}. The family of dynamics on the right-hand side is thus equivalent to the family of armed-bandit problems.
    
    \end{proof}

    The following definitions and affirmations are classical results from VC dimension theory and can be found in the reference \citet{shalev2014understanding}.
    \begin{definition}{Hypothesis class.}\\
        A set $\mathcal{H}$ is a hypothesis class if $\mathcal{H}$ is a set composed of functions from some domain $\X$ to $\{0,1\}$.
    \end{definition}
    \begin{definition}{Shattering.}\\
        A finite set $S\subseteq \X$ is shattered by a hypothesis class $\mathcal{H}$ if for any function $f$ from $S$ to $\{0,1\}$, there exists $h\in\mathcal{H}$ such that for all $x\in S$ we have $f(x)=h(x)$.
    \end{definition}
    \begin{definition}{VC-dimension.}\\
        Given a set $\X$, the VC-dimension of a hypothesis class defined on domain $\X$ is the maximal integer $d$ such that there exists a subset of $\X$ of size $d$ which is shattered by $\mathcal{H}$.
    \end{definition}
    \begin{proposition}{VC-dimension of the union.}\label{prop:VC-union}\\
        Given $r$ hypothesis classes of VC-dimension at most $d$ sharing the same domain, the VC-dimension of the union of these classes is at most $4d\log(2d) + 2\log(r)$.
    \end{proposition}
    \begin{proposition}{Agnostic PAC-learnability from bounded VC-dimension.}\label{prop:agnostic-PAC-learning}\\
        There exists a constant $C_1\in\R$ such that for any $\delta,\epsilon\in(0,1)$, given a hypothesis class $\mathcal{H}$ of VC-dimension $d$, and a dataset $S\subset\X\times\{0,1\}$ of size at least $C_1\frac{d+\log(\nicefrac{1}{\delta})}{\epsilon^2}$ composed of identically and independently drawn samples $(x,y)$ according to a probability measure $P$ on $\X\times\{0,1\}$, we have, with probability at least $1-\delta$ (over the sampled dataset),
        \begin{equation*}
            L_P(\argmin_{h\in\mathcal{H}} L_S(h)) \leq \min_{h'\in\mathcal{H}} L_P(h') + \epsilon,
        \end{equation*}
        where $L_P(h)$ is the error rate of classification according to measure $P$, $\mathbb{E}_{(x,y)\sim P}[\delta(h(x)\neq y)]$, and $L_S(h)$ is the error rate on the dataset, $\frac{1}{\abs{S}}\sum_{(x,y)\in S}\delta(h(x)\neq y)$.
    \end{proposition}

    \begin{lemma}\label{lem:GC-efficient}
        For Algorithm \ref{def:alg_working} with parameter $\alpha=1$ and any $\delta\in(0,1)$ on the family presented in the proof of Theorem \ref{thm:main_thm}, there exists $m$ a polynomial in $H$ and $\nicefrac{1}{\delta}$, such that with a dataset of trajectories larger than $m(H,\nicefrac{1}{\delta})$ obtained by $\pi^U$, with probability at least $1-\delta$ over the sampled dataset, for all time steps $1\leq t\leq H-2$ and with $s_t^c=1$, the learned function $f^t(a|\,s_t,g)$ perfectly predicts the action to reach the goal $g$, if the state $g$ is reachable from $s_t$.
    \end{lemma}
    \begin{proof}
    
    The proof can be decomposed in 4 main affirmations:
    \begin{enumerate}
        \item The classes of learning hypotheses have a VC-dimension upper-bounded by a polynomial in $H$.
        \item For all the problems in the family, there exists a hypothesis with a low rate of errors.
        \item A sufficiently low error rate implies a perfect accuracy on the predicted action to reach $g$ when $s_t^c=1$.
        \item Set the last three points together to entail the Lemma.
    \end{enumerate}

    We pose $s_t$ the current state, $s_H$ the state at the end of the trajectory after observing $s_t$. Thus, the learned function takes as input the vector $\begin{bmatrix} s_t\\ s_H \end{bmatrix}$ (concatenation of the real parts of the states without the time steps).

    \paragraph{1.} For $\alpha=1$, the class of hypotheses is the union of $O(H)$ sets of linear functions over 1 real-valued variable composed with a threshold. Those linear functions with a threshold have thus VC-dimension $2$. Using Proposition \ref{prop:VC-union}, we infer that the VC-dimension of our hypothesis class is in $O(\log H)$.

    \paragraph{2.} Take the hypothesis that selects the variable $s_H^{b,u,t}$ and uses the linear identity function composed with a threshold $\delta(s_H^{b,u,t}>0)$.

    Both sides of the environment ($s^c=0$ or $1$) have $\nicefrac{1}{2}$ probability. Since all the samples on the right-hand side are correctly predicted by this hypothesis and the samples on the left-hand side have $\nicefrac{1}{2}$ probability to have either action, this hypothesis has an average error rate of $\nicefrac{1}{4}$ on the distribution induced by $\pi^U$.

    \paragraph{3.} For all the possible feature selections, there is a finite set of  possible inputs, for example, $\{0,1\}$ for $s_H^{b,u,i}$, or $\{0\}$ for $s_t^{b,d,i}$. Conditioned on $s_t^c=1$, these inputs all have a constant probability of being sampled by $\pi^U$ which is independent of $H$ and $b$ the family parameters. We note the minimum on these probabilities $p^{\text{input}}_{\min}$, which is thus independent of the family parameters $H$ and $b$.
    
    Similarly, we note $p^{\text{output}}_{\min}$ the minimal non-zero probability under $\pi^U$ of any of the two actions conditioned on an input coming from a trajectory on the right-hand side sub-problem after feature selection.

    Suppose that the error rate under the distribution induced by $\pi^U$ of hypothesis $h$ for the goal-conditioned prediction problem is $\nicefrac{1}{4}+\epsilon$. If $\epsilon<\nicefrac{1}{2}\cdot p^{\text{input}}_{\min}\cdot p^{\text{output}}_{\min}$ then the hypothesis does not make any errors on the right-hand side sub-problem. By contradiction, suppose that it does make a mistake on a pair of input-output, then take the value of the input of the variable selected by the hypothesis. There is at least $\nicefrac{1}{2}\cdot p^{\text{input}}_{\min}$ samples with that input and a $p^{\text{output}}_{\min}$ proportion of them with the same output. The hypothesis will thus make a mistake on a set of input coming from the right-hand side sub-problem of probability at least $\nicefrac{1}{2}\cdot p^{\text{input}}_{\min}\cdot p^{\text{output}}_{\min}$. To which we have to add a $\nicefrac{1}{2}$ error rate on the trajectories on the left-hand side of the dynamics which have probability $\nicefrac{1}{2}$.
    
    \paragraph{4.} To generalize the result to all the time steps at the same time, we use a union bound over the probability of failure of each time step.

    Using Proposition \ref{prop:agnostic-PAC-learning} and the affirmations just proven, with any constant $\epsilon< \nicefrac{p^{\text{input}}_{\min}\cdot p^{\text{output}}_{\min}}{2}$ and $\delta=\nicefrac{\delta}{H}$, there exists $m(H,\delta)\in O(\log \frac{H}{\delta})$ such that the algorithm will output a hypothesis which makes no mistake on the right-hand side sub-problem.
    \end{proof}

    \begin{definition}\label{def:armed-bandit}
        A \emph{deterministic armed-bandit problem} is composed of a finite set of arms $\A$ and a function $f:\A\rightarrow [0,1]$ which associates an arm to a deterministic reward. 
    \end{definition}
    \begin{proposition}\label{prop:bandit-hard}
       Let be a family of deterministic armed-bandit problems defined by $A\in\N$ and $1\leq i^*\leq A$. The deterministic armed-bandit problem defined by $A$ and $i^*$ has $A$ arms, the arm $i^*$ gives a reward of $1$, the others give a reward of $0$.

       For any randomized algorithm that tries to identify the arm with the maximal associated reward with $o(A)$ calls to the reward associating function, there exists an $A$ large enough and a $i^*$ such that the algorithm fails with probability at least $\nicefrac{1}{3}$ on the problem defined by $A$ and $i^*$.  
    \end{proposition}
    \begin{proof}
        By contradiction, let's suppose that there exists a randomized algorithm $\mathrm{Alg}$ that solves any problem in the family with probability at least $1-\delta$ for some $\delta\in(0,1)$ using $K\in o(A)$ samples. Let us note $r\sim R$ the random variable following some probability distribution $R$ upon which the algorithm depends, $p$ the problem in input that it solves, and $\mathrm{Alg}(p,r)$ the output of the algorithm with $p$ and $r$ in input. For any $A$, let $U(A)$ be the uniform probability distribution upon $\{1,\ldots,A\}$ and $p^A_{i^*}$ the bandit problem defined by $A$ and $i^*$. We have for any $A$
        \begin{equation}
            \E_{i^*\sim U(A)}\E_{r\sim R}[\delta(\mathrm{Alg}(p^A_{i^*},r)=i^*)]\geq 1-\delta.
        \end{equation}
        Which implies
        \begin{equation}
            \E_{r\sim R}\E_{i^*\sim U(A)}[\delta(\mathrm{Alg}(p^A_{i^*},r)=i^*)]\geq 1-\delta.
        \end{equation}
        There must exist some $r$ which performs at least as well as the mean. Thus there exists a deterministic algorithm $\mathrm{Alg}_d$ such that $\E_{i^*\sim U(A)}[\mathrm{Alg}_d(p^A_{i^*})]\geq 1-\delta$.

        We suppose w.l.o.g. that this algorithm uses its $K\in o(A)$ tries in the set of the $K$ first arms. Then with probability $1-\frac{K}{A}$ it only receives $0$ rewards and must make a guess for the best arm with no information for the $A-K$ left untested arms. Thus the deterministic algorithm has a probability of failure lower bounded by $1-\frac{K}{A}-\frac{1}{A-K}$.

        Since $K\in o(A)$, there exists $A$ large enough such that this quantity is larger than $\delta$. Thus we have a contradiction.
    \end{proof}

\section{Classical RL methods implemented with the interface}\label{sec:translation-alg}
    We showcase how several classical RL methods can be implemented in our framework. We implement an Actor-Critic method, defined in Algorithm \ref{alg:pi-learning-original}, with the interface in Algorithm \ref{alg:pi-learning-interface}.
    
    In Algorithm \ref{alg:tree-search}, we implement a local planning tree search procedure relying on the capacity of the interface to generate transitions at any given state.
    
    This local search procedure can be combined with a model-free RL method, such as the defined Actor-Critic method, to replace the rewards by value functions. This construction allows to cast various classical Model Predictive Control methods to respect our Assumption \ref{assum:RL-class-with-interface}. 
    
    The combination of tree search and model-free methods can also be extended to methods that leverage feedback from results of tree search procedures to train neural networks \cite{feinberg2018model}, then possibly improve the tree search procedure with these neural networks in a virtuous cycle \cite{anthony2017thinking,silver2017mastering}.

    \begin{algorithm}[t]
        \caption{Deep Policy Gradient with Value function learning (Actor-Critic).}
        \label{alg:pi-learning-original}
        \textbf{Parameters:} $H$: the horizon of the MDP, $\Sc$: state space, $\A$: action space, $K$: number of sampled trajectories, $\eta>0$: scalar factor for gradient descent, $I$: number of iterations by initialization of the value function (should be divisible by $H$)
        \begin{algorithmic}
            \Require $P$: the operator corresponding to the MDP to solve
            \State Initialize a neural network $\pi_{\theta_1}:\Sc\rightarrow\Delta(\A)$
            \State Initialize a neural network $V_{\theta_2}:\Sc\rightarrow \R$
            \For{$k\leftarrow 1\ldots K$}
                \If{$k\cdot H\mod{I}=0$}
                    \State $\Bar{V}\leftarrow V_{\theta_2}$
                    \State Initialize a neural network $V_{\theta_2}:\Sc\rightarrow \R$
                \EndIf
                \State $s\leftarrow P(s|\,\bot)$
                \For{$t\leftarrow 0,\ldots, H-1$}
                    \State $a\sim \pi_{\theta_1}(a|\,s)$
                    \State $r,s' \leftarrow P(r,\,s'|\,s,\,a)$
                    \State $\theta_1\leftarrow \theta_1 + \eta (r+\Bar{V}(s')) \nabla_{\theta_1}\log \pi_{\theta_1}(a|\,s) $
                    \State $\theta_2\leftarrow \theta_2 - \eta \nabla_{\theta_2} \left[V_{\theta_2}(s)-r-\Bar{V}(s')\right]^2 $
                    \State $s\leftarrow s'$
                \EndFor
            \EndFor
            \State \Return $\pi_{\theta_1}(a|\,s)$
        \end{algorithmic}
    \end{algorithm}
    
    \begin{algorithm}[t]
        \caption{Deep Policy Gradient with Value function learning (Actor-Critic) implemented with the interface defined in Algorithm \ref{alg:interface}. The function $\mathrm{chunk}(D,\,I)$ partitions the data sequence $D$ into contiguous subsequences of length $I$ until there is less than $I$ elements left which are put in the last subsequence. We simplify the learning of $V$, the value function, following the discussion in Section \ref{sec:sym-nn}.}
        \label{alg:pi-learning-interface}
        \textbf{Parameters:} $H$: horizon of the MDP, $\Sc$: state space, $\A$: action space, $K$: number of sampled trajectories, $\eta>0$: scalar factor for gradient descent, $I$: number of iterations by initialization of the value function
        \begin{algorithmic}
            \Require $\mathrm{Env}$: the MDP interface with its methods defined in Algorithm \ref{alg:interface}
            \State Initialize a neural network $\pi_{\theta_1}:\Sc\rightarrow\Delta(\A)$
            \Function{$\mathrm{learningV}$}{$s,\,D$}
                \State Initialize a neural network $V_{\theta_2}:\Sc\rightarrow \R$
                \State $D_1,\ldots,D_{n-1},D_n\leftarrow\mathrm{chunk}(D,\,I)$
                \For{$(s,(r,V_{s'}))\leftarrow D_{n-1}$}
                    \State $\theta_2\leftarrow \theta_2 - \eta \nabla_{\theta_2} \left[V_{\theta_2}(s)-r-V_{s'}\right]^2 $
                \EndFor
                \State \Return $V_{\theta_2}(s')$
            \EndFunction
            \State $\F\leftarrow [\mathrm{learningV}]$
            \State $g\leftarrow g(a,r,f_{s'})=(r,f_{s'})$
            \For{$k\leftarrow 1\ldots K$}
                \State $s,\Bar{s},V_s\leftarrow\mathrm{Env}(\bot)$
                \For{$t\leftarrow 0\ldots H-1$}
                    \State $a\sim \pi_{\theta_1}(a|\,s)$ 
                    \State $s,a,r,\Bar{s},V_{s'} \leftarrow\mathrm{Env}(\Bar{s},\, a,\,g,\, \F)$
                     \State $\theta_1\leftarrow \theta_1 + \eta (r+V_{s'})\nabla_{\theta_1}\log \pi_{\theta_1}(a|\,s) $
                \EndFor
            \EndFor
            \State \Return $\pi_{\theta_1}(a|\,s)$
        \end{algorithmic}
    \end{algorithm}

    \begin{algorithm}[t]
        \caption{Tree-search implementation using the interface defined in Algorithm \ref{alg:interface}. This algorithm can be combined with the other methods presented, to use value function estimates instead of rewards for example.}
        \label{alg:tree-search}
        \textbf{Parameters:} $H$: horizon of the MDP, $\A$: action space, $H_S>0$: search horizon of the tree
        \begin{algorithmic}
            \Require $\mathrm{Env}$: the MDP interface with its methods defined in Algorithm \ref{alg:interface}
            \Function{Tree\_search}{$\Bar{s}$, $h$}
                \If{$h=0$}
                    \Return \_\,, $0$
                \EndIf
                \For{$a\in \A$}
                    \State $s,a,r,\Bar{s'},\_\leftarrow\mathrm{Env}(\Bar{s},\,a,\,\_,\,[])$
                    \If{$r=1$}
                        \State \Return $a$, $1$
                    \Else
                        \State $a',R\leftarrow$ Tree\_search($\Bar{s'}$, $h-1$)
                        \If{$R=1$}
                            \State \Return $a$, $1$
                        \EndIf
                    \EndIf
                \EndFor
                \State \Return \_\,, $0$
            \EndFunction
            \Function{Tree\_search}{$s_t$}
                \State $a,r \leftarrow$Tree\_search($\mathrm{Env\_encode}(s_t)$, $\min\{H_S,H-t\}$)
                \State \Return $a$
            \EndFunction
            \State $\pi(a|\,s_t)\leftarrow$ Tree\_search($s_t$)
            \State \Return $\pi(a|\,s_t)$
        \end{algorithmic}
    \end{algorithm}
        
\section{Symmetries in neural networks learning }\label{sec:sym-nn}
    We provide here a formal argument that classical neural network training procedures satisfy the condition put on functions contained in $\F$ described in Assumption \ref{assum:RL-class-with-interface}.

    We prove that the distribution of outputted functions by a gradient descent procedure respects such a symmetry.

    In Definition \ref{def:sym-cond}, we construct a symmetry condition that will help us to state and to prove our result. In Theorem \ref{thm:learn_sym}, we state the main result of this section on the symmetry of the learning procedure. Corollary \ref{coro:learn-sym} restate the Theorem, but in a form which can directly be compared with the condition on $\F$ stated in Assumption \ref{assum:RL-class-with-interface}.

    For the clarity of the exposition, we will abusively refer to the probability of a function instead of the probability of events in the related implicit $\sigma$-algebra. 

    \begin{definition}{Symmetry condition.}\label{def:sym-cond}
    
        Let $L$ be a function from a sequence of data points in $(\R^n\times\R)^N$ for some natural number $N$, to a set of distributions over functions $\Delta(\{\R^n\rightarrow\R\})$.
    
        For any sequence $d=((x_0,y_0),\ldots)$ of data points and for any permutation of the coordinates $p$ over $\R^n$ (Definition \ref{def:perm}), we define $d_p=((p(x_0),y_0),(p(x_1),y_1),\ldots)$.

        The function $L$ satisfies the symmetry condition if, for any permutation of the coordinates $p$, we have $L(d)=L(d_p)\circ p^{-1}$. In other words, for any function $h:\R^n\rightarrow\R$, the output with $d$ must give the same probability on $h$ as the output with $d_p$ on $h\circ p$.
    \end{definition}

    \begin{algorithm}[ht]
        \caption{Neural network training for a neural network with parameters $\theta=(W\in\R^{m\times n},\theta'\in\R^k)$ for some $n,m,k\in\N$, which is interpreted in $\mathrm{nn}_\theta(x)=q_{\theta'}(Wx)$ for $q$ that outputs a real and is smooth w.r.t. $\theta'$ and its input.}
        \label{alg:nn_with_linear_training}
        \textbf{Parameters:} $\eta>0$: a scalar factor for gradient descent, loss$:\R\times\R\rightarrow\R$: a smooth loss function w.r.t its inputs
        \begin{algorithmic}
            \Require $D$: a sequence of couples $(x\in\R^n,y)$
            \State Initialize $\theta'$ and matrix $W$. The elements of the matrix $W$ are initialized i.i.d.
            \For{$(x,y)\in D$}
                \State $\theta \leftarrow \theta - \eta \nabla_\theta \text{loss}(q_{\theta'}(Wx), y)$
            \EndFor\\
            \Return $\mathrm{nn}_{\theta=(W,\theta')}$
        \end{algorithmic}
    \end{algorithm}

    \begin{theorem}\label{thm:learn_sym}
        The output of the randomized Algorithm \ref{alg:nn_with_linear_training}, which produces a learned function given a dataset, follows a distribution defined by the dataset. The function between the dataset and the distribution of learned functions respects the symmetry condition.
    \end{theorem}
    \begin{proof}
        We prove by induction that for any permutation of the coordinates $p$ both the process that constructs the linear operator represented by $W$ and the processes that produce all the other functions of the input in the neural networks (output, gradients) satisfy the symmetry condition.

        Let $P$ be the permutation matrix corresponding to $p$.

        At initialization, the symmetry condition is satisfied everywhere. The matrix $W$ is as probable as $WP^{-1}$ and is independent of the dataset. All the other functions computed in the neural network depend on the input only through the linear operator that satisfies the condition, such that they also satisfy the symmetry condition.

        Now we look at the updates. Let $g$ be the gradient at the output of the linear operator for some input $x$. Then the update of $W$ is
        \begin{equation}
            W\leftarrow W -\eta gx^T.
        \end{equation}
        On the transformed dataset, with $WP^{-1}$ the matrix of the linear operator, $Px$ the new input, and $h$ the new gradient, the update becomes
        \begin{equation}
            WP^{-1}\leftarrow WP^{-1} -\eta h(Px)^T,
        \end{equation}
        which reduces to $(W-\eta hx^T)P^{-1}$. By the inductive hypothesis, we know that the processes that produce $h$ respect the symmetry condition, thus the gradient $h$ is sampled from the same distribution as $g$. Also, by the inductive hypothesis, the previous $WP^{-1}$ is as probable as the previous $W$ under their respective distributions. Consequently, the updated $W$ is as probable as the updated $WP^{-1}$.

        Similarly to the initialization, the processes that produce the rest of the network functions of the input respect the symmetry condition because they depend on the input only through the output of the linear operator since the start of the algorithm.
        
    \end{proof}

    \begin{corollary}\label{coro:learn-sym}
        The expectation of the distribution of the output of neural networks applied on some input $z\in\R^n$ produced by randomised Algorithm \ref{alg:nn_with_linear_training} for some sequence of points $d$ can be written as $f(z,((x_0,y_0),\ldots))$ for some function $f$ that is invariant to a permutation of the coordinates applied to $z$ and $x_0,\ldots$
    \end{corollary}
    \begin{proof}
        By construction, $f$ exists and is simply the function representing the process producing the distribution of learned functions composed with the expectation of their evaluations. Since the distribution of learned functions respects the symmetry condition by Theorem \ref{thm:learn_sym}, we know that for any permutation of the coordinates $p$, we have $\mathbb{E}_{\textbf{nn}\sim\textbf{Alg}(d)}[\textbf{nn}(z)]=\mathbb{E}_{\textbf{nn}\sim\textbf{Alg}(d_p)}[\textbf{nn}(p(z))]$ where \textbf{nn} is the neural network trained by the algorithm \textbf{Alg}.
    \end{proof}

\section{Numerical experiments}\label{sec:num-exp}

    We describe the implementation of the different methods we test: a goal-conditioned method in Algorithm \ref{alg:neural-gc}, fitted Q-iteration \cite{riedmiller2005neural} in Algorithm \ref{alg:fitted-Q-iter-implem} and Proximal Policy Optimization (PPO) \cite{schulman2017proximal}.

    We test all these algorithms with MLP neural networks. Any state in input to a neural network $s=(t,x)\in\Sc=\{(t,x)\in\{0,\ldots, H\}\times\R^n\}$ will be represented with a one-hot encoding of the time step $t$ concatenated to the real part $x$.

    Some notations, the function $\mathrm{CE}:\Delta(X)\times X\rightarrow \R$ (Cross-Entropy) computes $\mathrm{CE}(q,x)=-\log q(x)$ for some set $X$. The function $\mathrm{clip}(x\in\R,a\in\R,b\in\R)$ computes $\min\{\max\{x,a\},b\}$. The method $\mathrm{copy}$ copies the object, it allows to fix its state.

    The optimizations are performed with mini-batches estimation of the gradient, and AdamW, a regularized version of Adam \cite{loshchilov2017decoupled}. Hyper-parameter optimization was performed with Bayesian Optimization to maximize success probability on problems with horizon $H=45$ for the goal-conditioned algorithm, $H=12$ for the fitted Q-iteration algorithm, and $H=17$ for the PPO algorithm.

    The goal-conditioned algorithm was run with $2$ hidden layers of $256$ units each, parameter $I=1000$, $30000$ steps of AdamW with learning rate $2\cdot10^{-2}$, weight decay $10^{-3}$, and $100$ batches.

    The fitted Q-iteration algorithm was run with $2$ hidden layers of $256$ units each, for $K=50$ iterations, with $I=1000$ sampled trajectories by iteration, $\epsilon=0.3$, $1000$ steps of AdamW by iteration, learning rate $3\cdot10^{-4}$, weigh decay $5\cdot10^{-5}$, each $1000$ trajectories samples were divided into $20$ batches.

    The PPO algorithm was run with $2$ hidden layers of $512$ units each, for $K=50$ iterations, with $I=1000$ trajectories sampled by iterations, clipping parameter $\epsilon=0.2$, $\beta=10^{-3}$, $500$ steps of AdamW by iteration, learning rate of $2\cdot10^{-3}$, weight decay $10^{-8}$ and $10$ batches of the sampled trajectories.

    It takes less than 3 days with a single GPU to reproduce the results in this article.

    \begin{algorithm}[th]
        \caption{Neural goal-conditioned algorithm.}\label{alg:neural-gc}
        \textbf{Parameters:} $\Sc$: state space, $\A$: action space, $I$: number of samples
        \begin{algorithmic}[1]
            \Require $P$: the operator corresponding to the MDP to solve
            \State $D\leftarrow\{\}$
            \For{$i\leftarrow 1,\ldots I$}
                \State $D\leftarrow D\cup\{(s_0,a_0,r_0,s_1,a_1,\ldots,s_H)\sim P^{\pi^U}\}$
            \EndFor
            \State $D\leftarrow \{(s_t,a_t,s_H)|\,(\ldots, s_t, a_t,\ldots,s_H)\in D\}$.
            \State Initialize neural network $f_\theta:\Sc\rightarrow\Delta(\A)$
            \State $g\leftarrow \argmax_{s_H} r\quad \text{s.t.}\quad (\ldots,r,s_H)\in D$.
            \State $\theta\leftarrow \argmin\limits_{\theta}\frac{1}{\abs{D}}\sum_{(s,a,s_H)\in D}\mathrm{CE}(f_\theta([s,s_H]),a)$
            \State \Return $\pi(a|s)=f_\theta(a|s,g)$
        \end{algorithmic}
    \end{algorithm}

    \begin{algorithm}[th]
        \caption{Fitted Q-iteration}
        \label{alg:fitted-Q-iter-implem}
        \textbf{Parameters:} $H$: horizon of the MDP, $\Sc$: state-space, $\A$: action space, $K$: number of iterations, $I$: number of trajectory samples by iteration, $\epsilon$: exploration parameter
        \begin{algorithmic}
            \Require $P$: the operator corresponding to the MDP to solve
            \State $\pi_{\epsilon}^{Q}(a|\,s)\leftarrow(1-\epsilon)\delta(a=\argmax_{a\in A}Q(s,a)) + \nicefrac{\epsilon}{2}\delta(a=0)+\nicefrac{\epsilon}{2} \delta(a=1)$
            \State Initialize a neural network $Q_\theta:\Sc\rightarrow \R^{\abs{\A}}$
            \State $D\leftarrow \{\}$
            \For{$k\leftarrow 1,\ldots, K$}
                \State $D'\leftarrow\{\}$
                \For{$i\leftarrow 1,\ldots, I$}
                    \State $D'\leftarrow D'\cup\{(s_0,a_0,r_0,s_1,\ldots,s_H)\sim P^{\pi_\epsilon^{Q_\theta}}\}$
                \EndFor
                \State $D\leftarrow D\cup D\{(s,a,r,s')|\,(\ldots, s, a,r,s',\ldots)\in D\}$ 
                \State $\Bar{Q}\leftarrow \mathrm{copy}(Q_{\theta})$
                \State $\theta\leftarrow\argmin\limits_{\theta} \frac{1}{\abs{D}}\sum\limits_{(s,a,r,s')\in D}\left[Q_\theta(s,a)-r-\max_{a\in\A}\Bar{Q}(s',a)\right]^2$
            \EndFor
            \State \Return $\pi^{Q_\theta}_0(a|\,s)$
        \end{algorithmic}
    \end{algorithm}

    \begin{algorithm}[th]
        \caption{Proximal Policy Optimization}
        \label{alg:ppo-implem}
        \textbf{Parameters:} $H$: horizon of the MDP, $\Sc$: state-space, $\A$: action space, $K$: number of iterations, $I$: number of trajectory samples by iteration, $\epsilon$: clipping parameter, $\beta$: weight of the entropy regularization
        \begin{algorithmic}
            \Require $P$: the operator corresponding to the MDP to solve
            \Function{$\mathrm{PPO\_loss}$}{$s,\,a,\,R|\,\pi_{\theta_1},\,\Bar{\pi},\,\Bar{V}$}
                \State $A\leftarrow R-\Bar{V}(s)$
                \State $r_{\pi}\leftarrow \frac{\pi_{\theta_1}(a|s)}{\Bar{\pi}(a|s)}$
                \State \Return $\min\{r_\pi A, \mathrm{clip}(r_{\pi},1-\epsilon,1+\epsilon)A\}+ \beta \mathrm{entropy}(\pi_{\theta_1}(.|s))$
            \EndFunction
            \State Initialize a neural network $\pi_{\theta_1}:\Sc\rightarrow \Delta(\A)$
            \State Initialize a neural network $V_{\theta_2}:\Sc\rightarrow \R$
            \For{$k\leftarrow 1,\ldots, K$}
                \State $D\leftarrow\{\}$
                \For{$i\leftarrow 1,\ldots, I$}
                    \State $D\leftarrow D\cup\{(s_0,a_0,r_0,s_1,\ldots,s_H)\sim P^{\pi_\epsilon^{Q_\theta}}\}$
                \EndFor
                \State $D\leftarrow \{(s_t,a_t,\sum_{j=t}^{H-1}r_j)|\,(\ldots, s_t, a_t,r_t,\ldots)\in D\}$ 
                \State $\Bar{\pi}\leftarrow \mathrm{copy}(\pi_{\theta_1})$
                \State $\Bar{V}\leftarrow \mathrm{copy}(V_{\theta_2})$
                        \State $\theta_1\leftarrow \argmax\limits_{\theta_1} \frac{1}{\abs{D}}\sum\limits_{(s,a,R)}\mathrm{PPO\_loss}(s,a,R|\,\pi_{\theta_1},\Bar{\pi},\Bar{V})$
                        \State $\theta_2\leftarrow\argmin\limits_{\theta_2}\frac{1}{\abs{D}}\sum\limits_{(s,a,R)\in D}(V_{\theta_2}(s)-R)^2$
            \EndFor
            \State \Return $\pi_{\theta_1}(a|\,s)$
        \end{algorithmic}
    \end{algorithm}

\end{document}